\documentclass{article}

\usepackage{amsmath,amsthm}

\usepackage[nonatbib,preprint]{neurips_2023}
\newenvironment{internallinenumbers}{}{}

\usepackage[utf8]{inputenc} %
\usepackage[T1]{fontenc}    %
\usepackage{hyperref}       %
\usepackage{url}            %
\usepackage{booktabs}       %
\usepackage{amsfonts}       %
\usepackage{nicefrac}       %
\usepackage{microtype}      %
\usepackage{xcolor}         %

\usepackage{graphicx}
\hypersetup{colorlinks=true, allcolors=blue}

\usepackage[most]{tcolorbox}

\usepackage{xparse}

\usepackage{xspace}

\usepackage[inline]{enumitem}

\usepackage{amssymb,amsfonts,amsmath}

\usepackage[mathscr]{eucal}

\usepackage{amsbsy}

\usepackage{stmaryrd}

\usepackage{bm}

\usepackage{braket}  

\usepackage{scalerel}

\DeclareMathOperator{\trace}{trace}

\DeclareMathOperator{\rank}{rank}

\NewDocumentCommand{\bigO}{g}{\mathcal{O}\IfValueT{#1}{(#1)}}

\NewDocumentCommand{\qtext}{O{\quad}O{#1}m}{#1\text{#3}#2}

\NewDocumentCommand{\Real}{o}{\IfValueTF{#1}{\Rmsiz{#1}}{\mathbb{R}}}

\NewDocumentCommand{\Tr}{s}{\IfBooleanTF{#1}{\vphantom{\intercal}}{\intercal}}

\NewDocumentCommand{\Vc}{O{} m !g t' t"}{%
  \bm{#1{\mathbf{\MakeLowercase{#2}}}}%
  \IfValueT{#3}{_{#3}}%
  \IfBooleanTF{#4}{^{\Tr}}{%
    \IfBooleanT{#5}{^{\Tr*}}}%
}

\NewDocumentCommand{\Mx}{O{} m !g t' t"}{
  \bm{#1{\mathbf{\MakeUppercase{#2}}}}%
  \IfValueT{#3}{_{#3}}%
  \IfBooleanTF{#4}{^{\Tr}}{%
    \IfBooleanT{#5}{^{\Tr*}}}%
}

\NewDocumentCommand{\Tn}{O{} m !g}{%
  \boldsymbol{#1{\mathscr{\MakeUppercase{#2}}}}%
  \IfValueT{#3}{_{#3}}%
}

\NewDocumentCommand{\Tm}{O{} m O{} !g t' t"}{
  \bm{#1{\mathbf{\MakeUppercase{#2}}}}_{\prn#3(#4)}%
  \IfBooleanTF{#5}{^{\Tr}}{%
    \IfBooleanT{#6}{^{\Tr*}}}%
}

\NewDocumentCommand{\megaLRexp}{mmmmmmmm}{%
  \IfBooleanTF{#1}{%
    \IfBooleanTF{#2}{%
      \IfBooleanTF{#3}{%
        \IfBooleanTF{#4}%
        {\Biggl#6 #8 \Biggr#7}%
        {\biggl#6 #8 \biggr#7}}%
      {\Bigl#6 #8 \Bigr#7}}%
    {\bigl#6 #8 \bigr#7}}%
  {\IfBooleanTF{#5}
    {\left#6 #8 \right#7}
    {#6 #8 #7}}%
}

\NewDocumentCommand{\prn}{ ssss t+ g d() t'}
{%
  \IfValueT{#6}{\megaLRexp{#1}{#2}{#3}{#4}{#5}{(}{)}{#6}}%
  \IfValueT{#7}{\megaLRexp{#1}{#2}{#3}{#4}{#5}{(}{)}{#7}}%
  \IfBooleanT{#8}{^{\Tr}}%
}

\NewDocumentCommand{\abs}{ ssss t+ g d||}
{%
  \IfValueT{#6}{\megaLRexp{#1}{#2}{#3}{#4}{#5}{\vert}{\vert}{#6}}%
  \IfValueT{#7}{\megaLRexp{#1}{#2}{#3}{#4}{#5}{\vert}{\vert}{#7}}%
}

\NewDocumentCommand{\crly}{ ssss t+ g }
{%
  \IfValueT{#6}{\megaLRexp{#1}{#2}{#3}{#4}{#5}{\{}{\}}{#6}}%
}

\NewDocumentCommand{\nrm}{ ssss t+ g }
{%
  \IfValueT{#6}{\megaLRexp{#1}{#2}{#3}{#4}{#5}{\|}{\|}{#6}}%
}

\NewDocumentCommand{\fnrm}{ ssss t+ g }
{%
  \IfValueT{#6}{\megaLRexp{#1}{#2}{#3}{#4}{#5}{\|}{\|_{F}}{#6}}%
}

\NewDocumentCommand{\sqr}{ ssss t+ g d[] }
{%
  \IfValueT{#6}{\megaLRexp{#1}{#2}{#3}{#4}{#5}{\lbrack}{\rbrack}{#6}}%
  \IfValueT{#7}{\megaLRexp{#1}{#2}{#3}{#4}{#5}{\lbrack}{\rbrack}{#7}}%
}

\NewDocumentCommand{\dsqr}{ ssss t+ g d[] }
{%
  \IfValueT{#6}{\megaLRexp{#1}{#2}{#3}{#4}{#5}{\llbracket}{\rrbracket}{#6}}%
  \IfValueT{#7}{\megaLRexp{#1}{#2}{#3}{#4}{#5}{\llbracket}{\rrbracket}{#7}}%
}

\NewDocumentCommand{\ceil}{ ssss t+ m }
{%
  \megaLRexp{#1}{#2}{#3}{#4}{#5}{\lceil}{\rceil}{#6}
}

\NewDocumentCommand{\floor}{ ssss t+ m }
{%
  \megaLRexp{#1}{#2}{#3}{#4}{#5}{\lfloor}{\rfloor}{#6}
}
\NewDocumentCommand{\dsub}{m m t_ m}{#1_{#2_{#4}}}

\NewDocumentCommand{\tttMlist}{m m m m m m m}{%
  #1_{1} #2 %
  \IfBooleanF{#5}{#1_{2} #2} %
  \IfBooleanTF{#6}{ %
    #1_{3} %
    \IfBooleanT{#7}{#2 #1_{4} } %
  }{ %
    \IfBooleanTF{#3}{\cdots}{\dots} #2 #1_{#4} %
  }%
}

\NewDocumentCommand{\tttRMlist}{m m m m m m m}{%
  \IfBooleanTF{#6}%
  {\IfBooleanT{#7}{#1_{4} #2} #1_{3} #2 #1_{2} #2 #1_{1}}%
  {#1_{#4} \IfBooleanF{#5}{#2 #1_{#4-1}} #2 \cdots  #2 #1_{1}}%
}

\NewDocumentCommand{\tttSlist}{m m m m m m}{%
  #1_{1} %
  #2 \IfBooleanTF{#3}{\cdots}{\dots}
  #2 #1_{#5-1} %
  \IfValueT{#6}{#2 #6} %
  #2 #1_{#5+1} %
  #2 \IfBooleanTF{#3}{\cdots}{\dots}
  #2 #1_{#4} %
}

\NewDocumentCommand{\tttRSlist}{m m m m m m}{%
  #1_{#4} %
  #2 \IfBooleanTF{#3}{\cdots}{\dots}
  #2 #1_{#5+1} %
  \IfValueT{#6}{#2 #6} %
  #2 #1_{#5-1} %
  #2 \IfBooleanTF{#3}{\cdots}{\dots}
  #2 #1_{1} %
}

\NewDocumentCommand{\miwc}{s s t! O{i} O{d}}{%
  \tttMlist{#4}{,}{\BooleanFalse}{#5}{#3}{#1}{#2}%
}

\NewDocumentCommand{\siwc}{O{k} O{i} O{d} g}{%
  \tttSlist{#2}{,}{\BooleanFalse}{#3}{#1}{#4}%
}

\NewDocumentCommand{\minc}{s s t! O{i} O{d}}{%
  \tttMlist{#4}{}{\BooleanTrue}{#5}{#3}{#1}{#2}%
}

\NewDocumentCommand{\sinc}{O{k} O{i} O{d} g}{%
  \tttSlist{#2}{}{\BooleanTrue}{#3}{#1}{#4}%
}

\NewDocumentCommand{\tttTsze}{>{\SplitArgument{3}{,}}m}{\tttTszeFour#1} %
\NewDocumentCommand{\tttTszeFour}{m m m m}{#1 \IfValueT{#2}{\times #2} \IfValueT{#3}{\times #3} \IfValueT{#4}{\times #4}} %

\NewDocumentCommand{\msiz}{s s t! O{n} O{d} !g !g}{%
  \IfValueTF{#6}%
  {\tttTsze{#6}\IfValueT{#7}{\times \cdots \times #7}}%
  {#4_1 \IfBooleanF{#3}{\times #4_2}%
    \IfBooleanTF{#1}{\times #4_3 \IfBooleanT{#2}{\times #4_4}}{\times \cdots \times #4_{#5}}}%
}

\NewDocumentCommand{\Rmsiz}{s s t! O{n} O{d} !g !g}{%
  \mathbb{R}^{\IfValueTF{#6}%
  {\tttTsze{#6}\IfValueT{#7}{\times \cdots \times #7}}%
  {#4_1 \IfBooleanF{#3}{\times #4_2}%
    \IfBooleanTF{#1}{\times #4_3 \IfBooleanT{#2}{\times #4_4}}{\times \cdots \times #4_{#5}}}}%
}

\NewDocumentCommand{\ssiz}{O{k} O{n} O{d} g}{%
  \tttSlist{#2}{\times}{\BooleanTrue}{#3}{#1}{#4}%
}

\NewDocumentCommand{\Rssiz}{O{k} O{n} O{d} g}{%
  \mathbb{R}^{\tttSlist{#2}{\times}{\BooleanTrue}{#3}{#1}{#4}}%
}

\NewDocumentCommand{\tttMdom}{>{\SplitArgument{3}{,}}m}{\tttMdomFour#1} %
\NewDocumentCommand{\tttMdomFour}{m m m m}{[#1] \IfValueT{#2}{\otimes [#2]} \IfValueT{#3}{\otimes [#3]} \IfValueT{#4}{\otimes [#4]}} %

\NewDocumentCommand{\mdom}{s s t! O{n} O{d} !g !g}
{
  \IfValueTF{#6}
  {\tttMdom{#6}\IfValueT{#7}{\otimes \cdots \otimes [#7]}}
  {[#4_1] \IfBooleanF{#3}{\otimes [#4_2]}
    \IfBooleanTF{#1}{\otimes [#4_3] \IfBooleanT{#2}{\otimes [#4_4]}}{\otimes \cdots \otimes [#4_{#5}]}}
}

\NewDocumentCommand{\sdom}{O{k} O{n} O{d} g}{%
  [#2_{1}] %
  \otimes \cdots
  \otimes [#2_{#1-1}] %
  \IfValueT{#4}{\otimes [#4]} %
  \otimes [#2_{#1+1}] %
  \otimes \cdots
  \otimes [#2_{#3}] %
}

\NewDocumentCommand{\tttMsummand}{>{\SplitArgument{1}{/}}m}{\tttMsummandTwo#1} %
\NewDocumentCommand{\tttMsummandTwo}{mm}{\IfValueTF{#2}{\sum_{#1=1}^{#2}}{\sum_{#1}}} %

\NewDocumentCommand{\msum}{s s t! O{i} O{n} O{d} >{\SplitList{,}}g !g}{%
  \IfValueTF{#7}%
  {\ProcessList{#7}{\tttMsummand}\IfValueT{#8}{\cdots\tttMsummand{#8}}}%
  {%
    \tttMsummand{#4_1/#5_1}%
    \IfBooleanF{#3}{\tttMsummand{#4_2/#5_2}}%
    \IfBooleanTF{#1}%
    {\tttMsummand{#4_3/#5_3} \IfBooleanT{#2}{\tttMsummand{#4_4/#5_4}}} %
    {\cdots\tttMsummand{#4_{#6}/#5_{#6}}} %
  }%
}

\NewDocumentCommand{\Ev}{ ssss t+ g d() d[] }
{\mathbb{E}%
  \IfValueT{#6}{\megaLRexp{#1}{#2}{#3}{#4}{#5}{[}{]}{#6}}%
  \IfValueT{#7}{\megaLRexp{#1}{#2}{#3}{#4}{#5}{(}{)}{#7}}%
  \IfValueT{#8}{\megaLRexp{#1}{#2}{#3}{#4}{#5}{[}{]}{#8}}%
}

\NewDocumentCommand{\Var}{ ssss t+ g d[] }
{\text{Var}%
  \IfValueT{#6}{\megaLRexp{#1}{#2}{#3}{#4}{#5}{[}{]}{#6}}%
  \IfValueT{#7}{\megaLRexp{#1}{#2}{#3}{#4}{#5}{[}{]}{#7}}%
}

\DeclareDocumentCommand{\Pr}{ ssss t+ g d() d[] D<>{\mspace{1mu}}}
{\mathbb{P}%
  \IfValueT{#6}{\megaLRexp{#1}{#2}{#3}{#4}{#5}{\{}{\}}{#9#6#9}}%
  \IfValueT{#7}{\megaLRexp{#1}{#2}{#3}{#4}{#5}{(}{)}{#9#7#9}}%
  \IfValueT{#8}{\megaLRexp{#1}{#2}{#3}{#4}{#5}{[}{]}{#9#8#9}}%
}

\NewDocumentCommand{\iid}{}{i.i.d.\@\xspace}

\usepackage[capitalize,nameinlink]{cleveref}

\crefformat{equation}{\textup{#2(#1)#3}}
\crefrangeformat{equation}{\textup{#3(#1)#4--#5(#2)#6}}
\crefmultiformat{equation}{\textup{#2(#1)#3}}{ and \textup{#2(#1)#3}}
{, \textup{#2(#1)#3}}{, and \textup{#2(#1)#3}}
\crefrangemultiformat{equation}{\textup{#3(#1)#4--#5(#2)#6}}%
{ and \textup{#3(#1)#4--#5(#2)#6}}{, \textup{#3(#1)#4--#5(#2)#6}}{, and \textup{#3(#1)#4--#5(#2)#6}}

\Crefformat{equation}{#2Equation~\textup{(#1)}#3}
\Crefrangeformat{equation}{Equations~\textup{#3(#1)#4--#5(#2)#6}}
\Crefmultiformat{equation}{Equations~\textup{#2(#1)#3}}{ and \textup{#2(#1)#3}}
{, \textup{#2(#1)#3}}{, and \textup{#2(#1)#3}}
\Crefrangemultiformat{equation}{Equations~\textup{#3(#1)#4--#5(#2)#6}}%
{ and \textup{#3(#1)#4--#5(#2)#6}}{, \textup{#3(#1)#4--#5(#2)#6}}{, and \textup{#3(#1)#4--#5(#2)#6}}

\usepackage{pgfplots}
\pgfplotsset{
  plot style/.style={every axis plot/.append style={#1}},
  bar style/.style={every axis plot post/.append style={#1}},
  axis style/.style={every axis/.append style={#1}},
  mark style/.style={every mark/.append style={#1}},
  numformat/.style={/pgf/number format/#1},
  xnumformat/.style={xticklabel style={numformat/.list={#1}}},
  ynumformat/.style={yticklabel style={numformat/.list={#1}}},
  lognumformat/.style={log plot exponent style/.append style=
    {numformat/.list={#1}}},
  scatter classes/.style={only marks,scatter,
    scatter src=explicit symbolic,scatter/classes={#1}},
  axis style={
    xlabel near ticks, ylabel near ticks, %
    every axis title shift=3pt, %
    scale only axis, %
    cycle list={}, %
  },
  plot style={mark style={solid}},
  table/col sep=comma, %
}

\newtheorem{theorem}{Theorem}[section]
\newtheorem{corollary}[theorem]{Corollary}
\newtheorem{lemma}[theorem]{Lemma}
\newtheorem{prop}[theorem]{Proposition}

\crefname{prop}{Proposition}{Propositions}
\Crefname{prop}{Proposition}{Propositions}

\theoremstyle{remark}
\newtheorem{remark}[theorem]{Remark}
\crefname{remark}{Remark}{Remarks}
\Crefname{remark}{Remark}{Remarks}

\newtheorem{assump}{Assumption}

\crefname{assump}{Assumption}{Assumptions}
\Crefname{assump}{Assumption}{Assumptions}

\newcounter{asseq}
\newcounter{temp}

\NewDocumentEnvironment{assumption}{o}{
  \IfValueTF{#1}{\begin{assump}[#1]}{\begin{assump}}
  \setcounter{temp}{\value{equation}}
  \setcounter{equation}{\value{asseq}}
   
}{
  \setcounter{asseq}{0}
  \setcounter{equation}{\value{temp}}  
  \end{assump}
}

\tcbset{
  colback=white,colframe=black,
  boxrule=0.5pt,
  left=1.0mm,right=1.0mm,top=1mm,bottom=0mm,
  oversize=0mm,
  sharp corners,
  before upper*=\begin{internallinenumbers},
  after upper*=\end{internallinenumbers},
}

\tcolorboxenvironment{theorem}{}
\tcolorboxenvironment{corollary}{}
\tcolorboxenvironment{lemma}{}
\tcolorboxenvironment{prop}{}
\tcolorboxenvironment{assumption}{}
\tcolorboxenvironment{remark}{}

\usepackage{fixme}
\fxsetup{
  status=draft,
  nomargin,
  inline,
  theme=color,
  silent,
  marginface=\linespread{1}\footnotesize,
}
\FXRegisterAuthor{tk}{tke}{TK}
\FXRegisterAuthor{rw}{rwe}{RW}

\usepackage{tagging}
\NewDocumentCommand{\vx}{}{\Vc{x}}
\NewDocumentCommand{\I}{}{\Mx{I}}
\NewDocumentCommand{\X}{}{\Mx{X}}
\NewDocumentCommand{\Y}{}{\Mx{Y}}
\DeclareDocumentCommand{\Xi}{}{\Mx{X}{0}} %
\NewDocumentCommand{\Yi}{}{\Mx{Y}{0}} %

\NewDocumentCommand{\RX}{}{\Mx{\Phi}{1}} %
\NewDocumentCommand{\RY}{}{\Mx{\Phi}{2}} %
\NewDocumentCommand{\Xtp}{m}{\Mx{X}{t\IfBooleanT{#1}{+1}}}
\NewDocumentCommand{\Ytp}{m}{\Mx{Y}{t\IfBooleanT{#1}{+1}}}
\NewDocumentCommand{\Ztp}{m}{\Mx{Z}{t\IfBooleanT{#1}{+1}}}
\NewDocumentCommand{\Xt}{t+}{\Xtp{#1}}
\NewDocumentCommand{\Yt}{t+}{\Ytp{#1}}
\NewDocumentCommand{\Zt}{t+}{\Ztp{#1}}
\NewDocumentCommand{\Z}{}{\Mx{Z}}
\NewDocumentCommand{\A}{}{\Mx{A}}

\NewDocumentCommand{\Ft}{t+t+}{f\prn(\Xtp{#1},\Ytp{#2})}
\NewDocumentCommand{\D}{m}{\nabla_{\!\scaleobj{0.6}{#1}}}
\NewDocumentCommand{\DXFt}{t+t+}{\D{\X} f\prn(\Xtp{#1},\Ytp{#2})}
\NewDocumentCommand{\DYFt}{t+t+}{\D{\Y} f\prn(\Xtp{#1},\Ytp{#2})}
\NewDocumentCommand{\DDXFt}{t+t+}{\D{\X}^2 f\prn(\Xtp{#1},\Ytp{#2})}
\NewDocumentCommand{\DDYFt}{t+t+}{\D{\Y}^2 f\prn(\Xtp{#1},\Ytp{#2})}
\NewDocumentCommand{\Normal}{}{\mathcal{N}}
\NewDocumentCommand{\ft}{t+O{t}}{f(\X_{#2\IfBooleanT{#1}{+1}},\Y_{#2})}
\NewDocumentCommand{\Gt}{t+O{t}}{\IfBooleanTF{#1}{\D{\Y}}{\D{\X}}f(\X_{#2\IfBooleanT{#1}{+1}},\Y_{#2})}

\NewDocumentCommand{\Rt}{t+t+O{t}}{%
  \IfBooleanTF{#1}%
  {\IfBooleanTF{#2}%
    {\Mx[\hat]{R}}%
    {\Mx[\bar]{R}}%
  }%
  {\Mx{R}}%
  {#3}%
}

\NewDocumentCommand{\Tft}{t+t+O{t}}{f(\X_{#3\IfBooleanT{#1}{+1}},\Y_{#3\IfBooleanT{#2}{+1}},\Zt)}
\NewDocumentCommand{\TGt}{t+t+O{t}}{%
  \IfBooleanTF{#1}{\IfBooleanTF{#2}{\D{\Z}}{\D{\Y}}}{\D{\X}}%
  f(\X_{#3\IfBooleanT{#1}{+1}},\Y_{#3\IfBooleanT{#2}{+1}},\Zt)%
}

\DeclareMathOperator{\Prob}{Prob}

\title{Convergence of Alternating Gradient Descent for Matrix Factorization}

\author{
  Rachel~Ward
  \\
  University of Texas\\
  Austin, TX \\
  \texttt{rward@math.utexas.edu} \\
  \And
  Tamara G.~Kolda \\
  MathSci.ai \\
  Dublin, CA \\
  \texttt{tammy.kolda@mathsci.ai} \\
}

\begin{document}

\maketitle

\begin{abstract}
  We consider alternating gradient descent (AGD) with fixed step size applied to the asymmetric matrix factorization objective.
  We show that, for a rank-$r$ matrix $\A \in \Real[m,n]$,
  $T = C \prn*(\frac{\sigma_1(\A)}{\sigma_r(\A)})^2 \log(1/\epsilon)$
  iterations of alternating gradient descent suffice to reach an $\epsilon$-optimal factorization 
  $\fnrm{ \A - \X{T}" \Y{T}' }^2 \leq \epsilon \fnrm{ \A}^2$   with high probability
  starting from an atypical random initialization. The
  factors have rank $d \geq r$ so that $\X{T}\in\Real[m,d]$ and $\Y{T} \in\Real[n,d]$, and mild overparameterization suffices for the constant  $C$ in the iteration complexity $T$ to be an absolute constant. 
  Experiments suggest that our proposed initialization is not merely of theoretical benefit, but rather significantly improves the convergence rate of gradient descent in practice. Our proof is conceptually simple: a uniform Polyak-\L{}ojasiewicz (PL) inequality and uniform Lipschitz smoothness constant are guaranteed for a sufficient number of iterations, starting from our random initialization.  Our proof method should be useful for extending and simplifying convergence analyses for a broader class of nonconvex low-rank factorization problems.
\end{abstract}

\section{Introduction}
\label{sec:introduction}

This paper focuses on the convergence behavior of alternating gradient descent (AGD) on the low-rank matrix factorization objective
  \begin{equation}
    \label{sqlossgen}
    \min f(\X,\Y) \equiv  \frac{1}{2}\fnrm{ \X \Y' - \A }^2
    \qtext{subject to} \X \in \Real[m,d], \Y \in \Real[n,d].
  \end{equation}
Here, we assume $m,n \gg d \geq r=\rank(\A)$.
While there are a multitude of more efficient algorithms for low-rank matrix approximation, this serves as a simple prototype and special case of more complicated nonlinear optimization problems where gradient descent (or stochastic gradient descent) is the method of choice but not well-understood theoretically.  Such problems  include low-rank tensor factorization using the GCP algorithm descent \cite{hong2020generalized}, a stochastic gradient variant of the GCP algorithm \cite{kolda2020stochastic}, as well as deep learning optimization.

Surprisingly, the convergence behavior of gradient descent for low-rank matrix factorization is still not completely understood, in the sense that there is a large gap between theoretical guarantees and empirical performance. We take a step in closing this gap, providing a sharp linear convergence rate from a simple asymmetric random initialization. 
Precisely, we show that if $\A$ is rank-$r$, then a number of iterations   $T = C { \frac{d}{(\sqrt{d}-\sqrt{r-1})^2} } \frac{\sigma_1^2(\A)}{\sigma^2_r(\A)} \log(1/\epsilon)$ suffices to obtain an $\epsilon$-optimal factorization with high probability. 
Here, $\sigma_k(\A)$ denotes the $k$th singular value of $\A$ { and $C>0$ is a numerical constant}.
To the authors' knowledge, this improves on the state-of-art convergence result in the literature \cite{jiang2022algorithmic}, which provides an iteration complexity $T = C \prn**( \prn*(\frac{\sigma_1(\A)}{\sigma_r(\A)})^3 \log(1/\epsilon) )$ for gradient descent to reach an $\epsilon$-approximate rank-$r$ approximation\footnote{We note that our results are not precisely directly comparable as our analysis is for alternating gradient descent whereas existing results hold for gradient descent. However, empirically, alternating and non-alternating gradient descent exhibit similar behavior across many experiments}.  

Our improved convergence analysis is facilitated by our choice of initialization of $\X_0, \Y_0$, which appears to be new in the literature and is distinct from the standard Gaussian initialization.  Specifically, for $\RX$ and $\RY$ independent Gaussian matrices, we consider an ``unbalanced'' random initialization of the form 
$\Xi \sim \frac{1}{\sqrt{\eta}}\A \RX$  and $\Yi \sim \sqrt{\eta} \RY$, where $\eta > 0$ is the step-size used in (alternating) gradient descent. 
A crucial feature of this initialization is that the columns of  $\Xi$ are in the column span of $\A,$ and thus by invariance of the alternating gradient update steps, the columns of $\Xt$ remain in the column span of $\A$ throughout the optimization.  Because of this, a positive $r$th singular value of $\Xt$ provides a Polyak-\L{}ojasiewicz (PL) inequality for the region of the loss landscape on which the trajectory of alternating gradient descent is guaranteed to be confined to, even though the matrix factorization loss function $f$ in \cref{sqlossgen} does not satisfy a PL-inequality globally.

By Gaussian concentration, the pseudo-condition numbers $\frac{\sigma_{1}(\Xi)}{\sigma_r(\Xi)} \sim \frac{\sigma_{1}(\A)}{\sigma_r(\A)}$ are comparable with high probability\footnote{The pseudo-condition number, $\frac{\sigma_{1}(\A)}{\sigma_r(\A)}$, is equivalent to and sometimes discussed as the product of the product of the spectral norms of the matrix and its pseudoinverse, i.e., $\nrm{\A} \nrm{\A^{\!\dagger}}$}; for a range of step-size $\eta$ and the unbalanced initialization $\Xi \sim \frac{1}{\sqrt{\eta}}\A \RX$  and $\Yi \sim \sqrt{\eta} \RY$, we show that $\frac{\sigma_{1}(\Xt)}{\sigma_r(\Xt)}$ is guaranteed to remain comparable to $\frac{\sigma_{1}(\A)}{\sigma_r(\A)}$ for a sufficiently large number of iterations $t$ that we are guaranteed a linear rate of convergence with rate $\prn*(\frac{\sigma_{r}(\Xi)}{\sigma_1(\Xi)})^2$.

The unbalanced initialization, with the particular re-scaling of $\Xi$ and $\Yi$ by $\frac{1}{\sqrt{\eta}}$ and $\sqrt{\eta},$ respectively, is not a theoretical artifact but crucial in practice for achieving a faster convergence rate compared to a standard Gaussian initialization, as illustrated in \cref{fig:exp4}. Also in  \cref{fig:exp4}, we compare empirical convergence rates to the theoretical rates derived in \cref{thm:main1} below, indicating that our rates are sharp and made possible only by our particular choice of initialization.  {The derived convergence rate of (alternating) gradient descent starting from the particular asymmetric initialization where the columns of  $\Xi$ are in the column span of $\A$ can be explained intuitively as follows: in this regime, the $\Xt$ updates remain sufficiently small  with respect to the initial scale of $\Xi$, while the $\Yt$ updates change sufficiently quickly with respect to the initial scale $\Yi$, that the resulting alternating gradient descent dynamics on matrix factorization follow the dynamics of gradient descent on the linear regression problem $\text{min } g(\Y) = \|  \Xi \Y^T - \A \|_{F}^2$ where $\Xi$ is held fixed at its initialization.}

We acknowledge that our unbalanced initialization of $\Xi$ and $\Yi$ is different from the standard Gaussian random initialization in neural network training, which is a leading motivation for studying gradient descent as an algorithm for matrix factorization.  The unbalanced initialization should not be viewed as at odds with the implicit bias of gradient descent towards a balanced factorization \cite{cohen2021gradient, wang2021large, ahn2022learning, chen2022gradient}, which have been linked to better generalization performance in various neural network settings. An interesting direction of future research is to compare the properties of the factorizations obtained by (alternating) gradient descent, starting from various (balanced versus unbalanced) initializations.

\begin{figure}[!ht]
  \centering
  \includegraphics{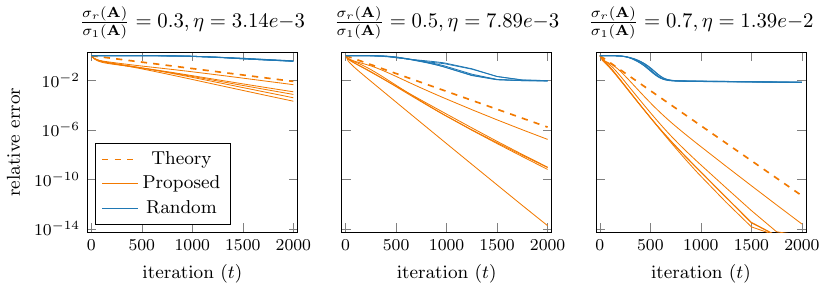}
  \caption{Alternating gradient descent for $\mathbf{A} \in \mathbb{R}^{100 \times 100}$ with $\text{rank}(\mathbf{A})=5$ and factors of size $100 \times 10$, The plot shows five runs each of our proposed initialization and compared with the standard random initialization. The title of each plot shows the condition and step length.}
  \label{fig:exp4}
\end{figure}

\section{Preliminaries}
\label{sec:matrix-case}

Throughout, for an $m \times n$ matrix ${\bf M}$,  $\| {\bf M} \|$ refers to the spectral norm and $\| {\bf M} \|_F$ refers to the Frobenius norm.

Consider the square loss applied to the matrix factorization problem \cref{sqlossgen}.
The gradients are
\begin{subequations}
  \label{eq:grad}
  \begin{align}
    \label{eq:grad-X}
    \D{\X} f(\X,\Y)  & = (\X\Y' - \A) \Y, \\
    \label{eq:grad-Y}
    \D{\Y} f(\X,\Y)  &= \prn(\X\Y' - \A)' \X.
  \end{align}  
\end{subequations}

We will analyze alternating gradient descent, defined as follows.
\begin{assumption}[Alternating Gradient Descent] \label{ass:agd}%
  For fixed stepsize $\eta > 0$ and initial condition $(\Xi, \Yi)$, the update is
  \begin{align}
    \label{eq:gd-Xt}
    \Xt+ &= \Xt - \eta \D{\X} f(\Xt,\Yt), \\
    \label{eq:gd-Yt}
    \Yt+ &= \Yt - \eta \D{\Y} f(\Xt+,\Yt).
  \end{align}%
\end{assumption}

We assume that the iterations are initialized in an asymmetric way, which depends on the step size $\eta$ and assumes a known upper bound on the spectral norm of $\A$.
The matrix factorization is of rank $d > r$, and we also make assumptions
about the relationship of $d$, $r$, and quantities $s$, $\beta$, and $\delta$ that will impact
the bounds on the probability of finding and $\epsilon$-optimal factorization.

\begin{assumption}[Initialization and key quantities]\label{ass:init}
  Draw random matrices $\RX,\RY \in \Real[n,d]$  with \iid $\Normal(0,1/d)$ and $\Normal(0,1/n)$ entries, respectively.
  Fix $C \geq 1$, $\nu < 1/3,$ and $D \leq \frac{C}{9}\nu,$ and let 
  \begin{equation}
    \Xi = \frac{1}{\eta^{1/2} \, C \, \sigma_1(\A)} \, \A \RX,
    \qtext{and} 
    \Yi = \eta^{1/2} \, D \, \sigma_1(\A) \, \RY.    
  \end{equation}
  The factor matrices each have $d$  columns with $r \leq d \leq n$.
  
  For $\tau > 0$, define
  \begin{equation}\label{eq:rho}
    \rho =  \tau \left( 1 - \frac{\sqrt{r-1}}{\sqrt{d}} \right)
  \end{equation}
  The number of iterations for convergence to $\epsilon$-optimal factorization
  will ultimately be shown to depend on
  \begin{equation}
    \label{b:define}
    \beta = \frac{ \rho^2 \sigma_r^2(\A)}{C^2 \sigma_1^2(\A)}.
  \end{equation}
  The probability of finding this $\epsilon$-optimal
  factorization will depend on
  \begin{equation}\label{eq:delta}
    \delta =  (C_1 \tau)^{d-r+1} + e^{-C_2 d} + e^{-r/2} + e^{-d/2} + {2 e^{-d}}
  \end{equation}
  where $C_1, C_2 > 0$ are the universal constants in \cref{prop:smallest}.
\end{assumption}

Observe that the initialization of $\Xi$ ensures its columns are in the column span of $\A$.

\begin{remark}\label{rmk:f0}
The quantity $\ft[0]$ does not depend on the step size $\eta$ or $\sigma_1(\A)$ in \cref{ass:init} since
\begin{displaymath}
 \Xi\Yi' - \A = \A\prn+(\frac{D}{C}\RX\RY^{\Tr} - \I).
\end{displaymath}  
\end{remark}

\section{Main results}

Our first main result gives a sharp guarantee on the number of iterations necessary for alternating gradient descent to be guaranteed to produce an $\epsilon$-optimal factorization. 
\begin{theorem}[Main result, informal]
\label{thm:main1}
For a rank-$r$ matrix $\A \in \mathbb{R}^{m \times n}$, set $d \geq r$ and
consider $\X_0, \Y_0$ randomly initialized as in \cref{ass:init}.
For any $\epsilon > 0$, there is an explicit step-size $\eta = \eta(\epsilon) > 0$ for
alternating gradient descent as in \cref{ass:agd} such that
\begin{displaymath}
  \fnrm{ \A - \X{T} \Y{T}' }^2 \leq \epsilon
  \qtext{for all}
  T \geq C  \frac{\sigma^2_1(\A)}{\sigma_r^2(\A)} \frac{1}{\rho^2} \log{\frac{\fnrm{ \A }^2}{\epsilon}}   
\end{displaymath}
with probability $1-\delta$ with respect to the draw of $\X_0$ and $\Y_0$
where $\delta$ is defined in \cref{eq:delta}.
Here, $C > 0$ is an explicit numerical constant. 
\end{theorem}
For more complete theorem statements, see \cref{cor:main} and \cref{thm:big} below.

We highlight a few points below.

\begin{enumerate}

\item The iteration complexity in \cref{thm:main1} is independent of the ambient dimensions $n, m$.  In the edge case $d = r$, $\frac{1}{\rho^2} = O(r^2)$, so the iteration complexity scales quadratically with $r$. With mild multiplicative overparameterization $d = (1+\alpha)r$, $\frac{1}{\rho^2} = \frac{(1+\alpha)}{(\sqrt{1+\alpha} - 1)^2}$, and the iteration complexity is essentially dimension-free.  This is a direct result of the dramatic improvement in the condition number of a $(1+\alpha)r \times r$ Gaussian random matrix compared to the condition number of a square $r \times r$ Gaussian random matrix.

    \item Experiments illustrate that initializing $\Xi$ in the column span of $\A$, and especially re-scaling $\X_0$ and $\Y_0$ by $\frac{1}{\sqrt{\eta}}$ and $\sqrt{\eta}$, respectively, is crucial in practice for improving the convergence rate of gradient descent.  See \cref{fig:exp1,fig:exp2,fig:exp3}.  

    \item The iteration complexity in \cref{thm:main1} is conservative.  In experiments, the convergence rate often follows a dependence on $\frac{\sigma_r(\A)}{\sigma_1(\A)}$ rather than $\frac{\sigma^2_r(\A)}{\sigma^2_1(\A)}$ for the first several iterations.

\end{enumerate}

\subsection{Our contribution and prior work}
The seminal work of Burer and Monteiro \cite{burer2003nonlinear, burer2005local} advocated for the general approach of using simple algorithms such as gradient descent directly applied to low-rank factor matrices for solving non-convex optimization problems with low-rank matrix solutions. Initial theoretical work on gradient descent for low-rank factorization problems such as \cite{zhao2015nonconvex}, \cite{tu2016low}, \cite{zheng2016convergence}, \cite{sanghavi2017local}, \cite{bhojanapalli2016dropping}
did not prove global convergence of gradient descent, but rather 
 local convergence of gradient descent starting from a spectral initialization (that is, an initialization involving SVD computations). In almost all cases, the spectral initialization is the dominant computation, and thus a more global convergence analysis for gradient descent is desirable.

 Global convergence for gradient descent for matrix factorization problems without additional explicit regularization was first derived in the symmetric setting, where $\A \in \mathbb{R}^{n \times n}$ is positive semi-definite, and $f(\X) = \fnrm{ \A - \X \X' }^2$, see for example  \cite{ge2015escaping, jain2017global, chen2019gradient}. 

For overparameterized symmetric matrix factorization, the convergence behavior and implicit bias towards particular solutions for gradient descent with small step-size and from small initialization was analyzed in the work \cite{gunasekar2017implicit, li2018algorithmic, arora2019implicit, chou2020gradient}. 

The paper \cite{ye2021global} initiated a study of gradient descent with fixed step-size in the more challenging setting of \emph{asymmetric} matrix factorization, where $\A \in \mathbb{R}^{m \times n}$ is rank-$r$ and the objective is $\fnrm{ \A - \X \Y' }^2$. This work improved on previous work in the setting of gradient flow and gradient descent with decreasing step-size \cite{du2018algorithmic}.
The paper \cite{ye2021global} proved an iteration complexity of $T = \mathcal{O}\prn*(nd\prn*(\frac{\sigma_1(\A)}{\sigma_1(\A)})^4\log(1/\epsilon))$ for reaching an $\epsilon$-approximate matrix factorization, starting from small i.i.d.\@ Gaussian initialization for the factors $\X_0, \Y_0$.  More recently, \cite{jiang2022algorithmic} studied gradient descent for asymmetric matrix factorization, and proved an iteration complexity $T = \mathcal{O}\prn*(C_d \prn*(\frac{\sigma_1(\A)}{\sigma_r(\A)})^3\log(1/\epsilon))$ to reach an $\epsilon$-optimal factorization, starting from small i.i.d.\@ Gaussian initialization. 

We improve on previous analysis of gradient descent applied to objectives of the form \cref{sqlossgen}, providing an improved  iteration complexity $T = \mathcal{O}\prn*(\prn*(\frac{\sigma_1(\A)}{\sigma_r(\A)})^2\log(1/\epsilon))$ to reach an $\epsilon$-approximate factorization. 
There is no dependence on the matrix dimensions in our bound, and the dependence on the rank $r$ disappears if the optimization is mildly over-parameterized, i.e., $d = (1+\alpha)r.$  
We do note that our results are not directly comparable to previous work as we analyze alternating gradient descent rather than full gradient descent.  
Our method of proof is conceptually simpler than previous works; in particular, 
because our initialization $\X_0$ is in the column span of $\A,$ we do not require a two-stage analysis and instead can prove a fast linear convergence from the initial iteration.

\section{Preliminary lemmas}

\begin{lemma}[Bounding sum of norms of gradients]
  \label{lemma1}
  Consider alternating gradient descent as in \cref{ass:agd}.
  If $\nrm{\Yt}^2 \leq \frac{1}{\eta}$, then 
  \begin{equation}
    \label{eq:lip}
    \fnrm{\Gt}^2 \leq \frac{2}{\eta} \prn*{ \ft - \ft+}.
  \end{equation}
  If moreover $\| \X_{t} \|^2 \leq \frac{2}{\eta},$ then $
 \ft[t] \leq f(\X_t, \Y_{t-1})$.
  Consequently, if $\nrm{\Yt}^2 \leq \frac{1}{\eta}$ for all $t=0,\dots,T,$ and $\nrm{\Xt}^2 \leq \frac{2}{\eta}$ for all $t=0,\dots,T,$  then $\sum_{t=0}^T \fnrm{\Gt}^2 \leq \frac{2}{\eta} \ft[0]$
  Likewise, if $\nrm{\Xt+}^2 \leq \frac{1}{\eta}$, then 
  \begin{equation}
    \label{eq:lip2}
    \fnrm{ \nabla_Y f(\Xt+, \Yt) }^2  \leq \frac{2}{\eta} (f(\Xt+, \Yt) - f(\Xt+, \Yt+)) .
  \end{equation}
   and if $\| \Y_{t} \|^2 \leq \frac{2}{\eta},$ then $
 f(\Xt+, \Yt) \leq f(\Xt, \Yt)$, and so  if $\nrm{\Xt+}^2 \leq \frac{1}{\eta}$ for all $t=0,\dots,T,$ and $\nrm{\Yt}^2 \leq \frac{2}{\eta}$ for all $t=0,\dots,T,$  then
$\sum_{t=0}^T \fnrm{\Gt+}^2 \leq \frac{2}{\eta} \ft[0].$
\end{lemma}

The proof of \cref{lemma1} is in \cref{sec:proof-lemma1}.

\begin{prop}[Bounding singular values of iterates]
  \label{prop:big}
  Consider alternating gradient descent as in \cref{ass:agd}.
  Set $f_0 := f(\X_0, \Y_0)$. Set    
  $T_* = \floor+{ \frac{1}{32 \eta^2 {f_0}} }.$
  Suppose $\sigma^2_{1}(\Xi) \leq  \frac{9}{16 \eta}, \sigma_{1}^2(\Yi) \leq  \frac{9}{16 \eta}.$  Then for all $0 \leq T \leq T_{*},$
  \begin{enumerate}
  \item  $\nrm{ \X{T} } \leq \frac{1}{\sqrt{\eta}} \qtext{and} \nrm{ \Y{T} }
    \leq \frac{1}{\sqrt{\eta}}$,
  \item $\sigma_{r}(\X_0)  -  \sqrt{2 T \eta {f_0}} \leq  \sigma_{r}(\X{T}) \leq \sigma_1(\X{T}) \leq \sigma_1(\X_0) +  \sqrt{2 T \eta {f_0}}$,
      \item $\sigma_{r}(\Y_0)  -  \sqrt{2 T \eta {f_0}} \leq  \sigma_{r}(\Y{T}) \leq \sigma_1(\Y{T}) \leq \sigma_1(\Y_0) +  \sqrt{2 T \eta {f_0}}$.
  \end{enumerate}
\end{prop}
The proof of \cref{prop:big} is in  \cref{sec:proof1}.

\begin{prop}[Initialization]
  \label{prop:initialize}
    Assume $\Xi$ and $\Yi$ are initialized as in \cref{ass:init}, which
    fixes $C\geq 1$, $\nu<1$, and $D \leq \frac{C}{9}\nu$, and
    consider alternating gradient descent as in \cref{ass:agd}.
    Then with probability at least $1 - \delta$,
    with respect to the random initialization and
    $\delta$ defined in \cref{eq:delta},
    \begin{enumerate}
    \item $\displaystyle\frac{1}{\sqrt{\eta}} \frac{\rho}{C} \frac{ \sigma_r(\A)}{\sigma_1(\A)}  \leq \sigma_r(\Xi), \quad \sigma_1(\Xi) \leq \displaystyle\frac{3}{C\sqrt{\eta}}, \quad \sigma_1(\Yi) \leq \displaystyle\frac{\sqrt{\eta} \, C \, \nu \, \sigma_1(\A)}{3}$,
    \item  $(1 - 3\nu)^2  \| \A \|_F^2 \leq f(\Xi, \Yi) \leq (1 + \nu)^2  \| \A \|_F^2.$      
    \end{enumerate}
  \end{prop}
  The proof of \cref{prop:initialize} is in  \cref{sec:proof2}.

Combining the previous two propositions gives the following.
\begin{corollary}
  \label{cor:big}
  Assume $\Xi$ and $\Yi$ are initialized as in \cref{ass:init}, with the stronger
  assumption that $C \geq 4$.
  Consider alternating gradient descent as in \cref{ass:agd} with $\eta \leq \frac{9}{4 C \nu \sigma_1(\A)}.$
  With $\beta$ as in \cref{b:define} and $f_0=f(\Xi,\Yi)$, set
   $ T = \floor***{   \frac{\beta}{8\eta^2 {f_0}} }.$
  With probability at least $1 - \delta$,
  with respect to the random initialization and
  $\delta$ defined in \cref{eq:delta}, the
  following hold for all $t=1, \dots, T$:
  $$
  \sigma_r(\Xt)  \geq \frac{1}{2}\sqrt{ \frac{\beta}{\eta}}, \text{ and } \sigma_1(\Xt),  \sigma_1(\Yt) \leq  \frac{3}{C\sqrt{\eta}} + \frac{1}{2}\sqrt{ \frac{\beta}{\eta}}
$$
\end{corollary}

\begin{proof}
By \cref{prop:initialize}, we have the following event occurring with the stated probability: 
\begin{displaymath}
\frac{\rho^2 \sigma_r^2(\A)}{C^2 \sigma_1^2(\A) \eta} \leq \sigma_r^2(\X_0) \leq \sigma_1^2(\X_0) \leq \frac{9}{16 \eta}
\end{displaymath}
where the upper bound uses that $C \geq 4$.
Moreover, using that $\eta \leq \frac{9}{4 C \nu \sigma_1(\A)}$,
$\sigma_1^2(\Y_0) \leq 9\eta C^2  \sigma_1(\A)^2  \leq \frac{9}{16 \eta}.$
For $\beta$ as in \cref{b:define}, note that $T = \floor***{ \frac{\beta}{8\eta^2 {f_0}}  }  \leq \floor***{ \frac{1}{32 \eta^2 {f_0}} },$
which means that we can apply \cref{prop:big} up to iteration $T$, resulting in the bound
$\sigma_r(\Xt) \geq \sigma_r(\Xi) - \sqrt{2 T \eta f_0} \geq \frac{1}{2} \sqrt{\frac{\beta}{\eta}}.$
Similarly, 
$\sigma_1(\Xt),  \sigma_1(\Yt) \leq  \frac{3}{C\sqrt{\eta}} + \frac{1}{2}\sqrt{ \frac{\beta}{\eta}}.$
\end{proof}

Finally, we use a couple crucial lemmas which apply to our initialization of $\Xi$ and $\Yi$.

\begin{lemma}
  \label{lemma:invariant}
  Consider alternating gradient descent as in \cref{ass:agd}.
  If $\text{ColSpan}(\Xi) \subseteq \text{ColSpan}(\A),$ then  $\text{ColSpan}(\Xt) \subseteq \text{ColSpan}(\A)$ for all $t$.
\end{lemma}

\begin{proof}
  Suppose $\text{ColSpan}(\Xt) \subseteq \text{ColSpan}(\A)$.
  Then $\text{ColSpan}(\Xt \Yt' \Yt) \subseteq \text{ColSpan}(\Xt) \subseteq \text{ColSpan}(\A)$ and 
by the update of \cref{ass:agd},
\begin{equation*}
\begin{aligned}[b]
\text{ColSpan}(\Xt+) &= \text{ColSpan} ({\Xt} + \eta \A \Yt - \eta \Xt \Yt' \Yt ) \\
&\subseteq \text{ColSpan} ({\Xt}) \cup \text{ColSpan}( \A \Yt) \cup \text{ColSpan}( \Xt \Yt' \Yt )  \\
&\subseteq \text{ColSpan} (\A). 
\end{aligned}\qedhere
\end{equation*}
\end{proof}

\begin{lemma}
  \label{PLlemma}
  If $\A$ is rank $r$, 
  $\text{ColSpan}(\Xt) \subseteq \text{ColSpan}(\A)$, and $\sigma_r(\Xt) > 0$ then
  \begin{align}
    \fnrm{ \D{\Y} f(\Xt, \Y{t-1}) }^2 &\geq  2 \sigma^2_r (\Xt) f(\Xt, \Y{t-1}).
  \end{align}  
\end{lemma}

\begin{proof}
  If $\A$ is rank $r$, $\text{ColSpan}(\Xt) \subseteq \text{ColSpan}(\A)$,
  and $\sigma_r(\Xt) > 0,$ then $\Xt$ is rank-$r$; thus,
  $\text{ColSpan}(\Xt) = \text{ColSpan}(\A)$.
  In this case, each column of $(\Xt \Yt' - \A)$ is in the row span of $\Xt'$, and so
  \begin{multline*}
    \fnrm{ \D{\Y} f(\Xt, \Y{t-1}) }^2
    = \fnrm{ \prn(\Xt \Y{t-1}' - \A)' \Xt }^2  \\
    = \fnrm{ \Xt' (\Xt \Y{t-1}' - \A) }^2  \geq \sigma_r^2(\Xt) \; \fnrm{\Xt \Y{t-1}' - \A }^2. 
\qedhere      
  \end{multline*}
\end{proof}

{
\remark{While \cref{lemma:invariant,PLlemma} are straightforward to prove in the setting we consider where $\A$ is exactly rank-$r$, these lemmas no longer hold beyond the exactly rank-$r$ setting (while the rest of the theorems we use do extend).  Numerical experiments such as \cref{fig:exp3} illustrate that the proposed algorithm does extend to finding best low-rank approximations to general matrices, but the theory for these experiments will require a careful reworking of \cref{lemma:invariant,PLlemma}.}}

\section{Main results}

We are now ready to prove the main results.

\begin{theorem}
  \label{thm:main}
  Assume $\Xi$ and $\Yi$ are initialized as in \cref{ass:init}, with the stronger
  assumption that $C \geq 4$.
  Consider alternating gradient descent as in \cref{ass:agd} with
  \begin{displaymath}
    \eta \leq \frac{9}{4 C \nu \sigma_1(\A)}.
  \end{displaymath} 
  With $\beta$ as in \cref{b:define} and $f_0=f(\Xi,\Yi)$, set
  \begin{displaymath}
    T = \floor***{   \frac{\beta}{8\eta^2 {f_0}} } .
  \end{displaymath}
  Then with probability at least $1 - \delta$,
  with respect to the random initialization and
  $\delta$ defined in \cref{eq:delta}, the
  following hold for all $t=1, \dots, T$:

\begin{align}
\label{RHS}
\fnrm{ \A - \Xt \Yt' }^2 &\leq 2\exp \left( - {\beta t}/{4} \right) f_0 \nonumber \\
&\leq  \exp \left( - {\beta t}/{4} \right) \left( 1 + \nu \right)^2 \fnrm{ \A }^2.
\end{align}
 
\end{theorem}
\begin{proof}%
  \cref{cor:big} implies that  $\sigma_{r}(\Xt)^2 \geq \frac{\beta}{4\eta}$ for $t=1, \dots, T$.   \Cref{lemma:invariant,PLlemma} imply since $\Xi$ is
  initialized in the column space of $\A$,  $\Xt$ remains in the column space of $\A$ for all $t$, and  
\begin{equation}
  \label{eq:pl}
  \begin{aligned}
 \fnrm{ \D{\Y} f(\Xt+, \Yt) }^2 &= \fnrm{ (\A' - \Yt \Xt+') \Xt+ }^2 \geq  \sigma_r(\Xt+)^2 \fnrm{ (\A' - \Yt \Xt+')}^2  \\
&\geq \frac{\beta}{4\eta} \fnrm{ \A - \Xt+ \Yt' }^2 = \frac{ \beta}{2\eta} f(\Xt+, \Yt).
\end{aligned}
\end{equation}
That is, a lower bound on $\sigma_{r}( \X_t )^2$ implies that the gradient step with respect to $\Y$ satisfies the Polyak-Lojasiewicz (PL)-equality\footnote{A function $f$ satisfies the PL-equality if for all $\vx \in \Real^m$, $f(\vx) - f(\vx^{*}) \leq \frac{1}{2m} \| \nabla f(\vx) \|^2$, where $f(\vx^{*}) = \min_{\vx} f(\vx)$}.

We can combine this PL inequality with the Lipschitz bound from \cref{lemma1} to derive the linear convergence rate. Indeed, by \cref{eq:lip}, 
\begin{align*}
  \label{standard}
  f(\Xt+, \Yt+) - f(\Xt+, \Yt) &\leq -\frac{\eta}{2}  \fnrm{ \D{\Y} f(\Xt+,\Yt) }^2 \leq - \frac{\beta}{4} f(\Xt+, \Yt) . 
\end{align*}
where the final inequality is \cref{eq:pl}.
Consequently, using \cref{prop:initialize},
\begin{multline*}
  f(\X_{T}, \Y_{T}) \leq (1 - {\beta}/{4}) f(\X_{T-1}, \Y_{T-1}) \leq \left(1 - \beta/4 \right)^{T} f(\Xi, \Yi) \\ \leq  \exp \left( - \beta T /4 \right) f(\Xi, \Yi). \nonumber
  \qedhere
\end{multline*}
\end{proof}

\begin{corollary}
\label{cor:main}
  Assume $\Xi$ and $\Yi$ are initialized as in \cref{ass:init}, with the stronger
  assumptions that $C \geq 4$ and $\nu \leq \frac{1}{2}$.
  Consider alternating gradient descent as in \cref{ass:agd} with
  \begin{equation}\label{eta:max}
    \eta \leq \frac{\beta}{\sqrt{32 f_0 \log(2f_0/\epsilon)}},
  \end{equation} 
  where $\beta$ is defined in \cref{b:define} and $f_0=f(\Xi,\Yi)$.
  Then with probability at least $1 - \delta$,
  with respect to the random initialization and
  $\delta$ defined in \cref{eq:delta}, it holds
  \begin{displaymath}
    \fnrm{ \A - \X{T}" \Y{T}' }^2 \leq \epsilon
    \qtext{at interation}
    T = \floor***{   \frac{\beta}{8\eta^2 {f_0}} }.
  \end{displaymath}
  Here $\rho$ is defined in \cref{eq:rho}.
  Using the upper bound for $\eta$ in \cref{eta:max}, the iteration
  complexity to reach an $\epsilon$-optimal loss value is
\begin{displaymath}
  T = \mathcal{O} \prn+( \prn+(\frac{\sigma_1(\A)}{\sigma_r(\A)})^2
  \frac{1}{\rho^2} \log\prn+( \frac{\fnrm{ \A }^2}{\epsilon} ) ) .
\end{displaymath}

\end{corollary}
This corollary follows from \cref{thm:main} by solving for $\eta$ so that the RHS of  \cref{RHS} is at most $\epsilon,$ and then noting that $\eta \leq \frac{\beta}{\sqrt{32 f_0 \log(2f_0/\epsilon)}} $ implies that $\eta \leq \frac{9}{4 C \nu \sigma_1(\A)}$ when $\nu \leq \frac{1}{2},$ using the lower bound on $f_0$ from \cref{prop:initialize}.

Using this corollary recursively, we can prove that the loss value remains small for $T' \geq  \lfloor \frac{\beta}{8 \eta^2 f_0} \rfloor$, provided we increase the lower bound on $C$ by a factor of 2. The proof is in \cref{sec:proof4}. 
\begin{corollary}
\label{thm:big}
  Assume $\Xi$ and $\Yi$ are initialized as in \cref{ass:init}, with the stronger
  assumptions that $C \geq 8$ and $\nu \leq \frac{1}{2}$.
  Fix $\epsilon < 1/16$, and 
  consider alternating gradient descent as in \cref{ass:agd} with
  \begin{equation}\label{eta:max2}
    \eta \leq \frac{\beta}{\sqrt{32 f_0 \log(1/\epsilon)}},
  \end{equation} 
  where $\beta$ is defined in \cref{b:define} and $f_0=f(\Xi,\Yi)$.
  Then with probability at least $1 - \delta$,
  with respect to the random initialization and
  $\delta$ defined in \cref{eq:delta}, it holds for any $k \in \mathbb{N}$ that
  \begin{align*}
    \fnrm{ \A - \X{T} \Y{T}' }^2 &\leq \epsilon^k \fnrm{ \A - \Xi \Yi' }^2
    &\qtext{for}T &\geq  \sum_{\ell=0}^{k-1}
    \floor+{ \prn***(\frac{1}{4 \epsilon})^{\ell} \frac{\beta}{8\eta^2 f_0 }}. 
  \end{align*}
\end{corollary}

\section{Numerical experiments }

We perform an  illustrative numerical experiment to demonstrate both the theoretical and practical benefits of the proposed initialization.
We use gradient descent \emph{without} alternating to demonstrate that this theoretical assumption makes little difference in practice.
We factorize a rank-5 ($r=5$) matrix of size $100 \times 100$. The matrix
is constructed as $\A = \Mx{U} \Mx{\Sigma} \Mx{V}'$ with $\Mx{U}$ and $\Mx{V}$ random $100 \times 5$ orthonormal matrices and singular value
ratio $\sigma_r(\A)/\sigma_1(\A) = 0.9$.
The same matrix is used for each set of experiments.
We compare four initializations:
\begin{align*}
  \text{Proposed:} &&
  \X{0} & = \frac{1}{\sqrt{\eta}\sqrt{d}C\sigma_1} \A \Mx{\Phi}{(n \times d)} &
  \Y{0} & = \frac{\sqrt{\eta}D\sigma_1}{\sqrt{n}} \Mx{\Phi}{(n \times d)}
  \\
  \text{ColSpan($\A$):} &&
  \X{0} & = \frac{1}{10\sqrt{d}} \A \Mx{\Phi}{(n \times d)} &
  \Y{0} & = \frac{1}{10\sqrt{n}} \Mx{\Phi}{(n \times d)}   
  \\
  \text{Random:} &&
  \X{0} & = \frac{1}{10\sqrt{m}} \Mx{\Phi}{(m \times d)} &
  \Y{0} & = \frac{1}{10\sqrt{n}} \Mx{\Phi}{(n \times d)}   
  \\
  \text{Random-Asym:} &&
  \X{0} & = \frac{1}{\sqrt{\eta}}\frac{1}{10\sqrt{m}} \Mx{\Phi}{(m \times d)} &
  \Y{0} & = \sqrt{\eta}\frac{1}{10\sqrt{n}} \Mx{\Phi}{(n \times d)}   
\end{align*}
Here, $\Mx{\Phi}$ denotes a random matrix with independent entries from $\Normal(0,1)$.
The random initialization is what is commonly used and analyzed.
We include ColSpan($\A$) to understand the impact of starting in the
column space of $\A$ for $\X{0}$.
Our proposed initialization (\cref{ass:init}) combines this with an asymmetric scaling.
In all experiments we use the defaults $C=4$ and $D=C\nu/9$ with $\nu=1e{-}10$
for computing the proposed initialization as well as the theoretical step size.
We assume $\sigma_1 = \sigma_1(\A)$ is known in these cases. (In practice,
a misestimate of $\sigma_1$ can be compensated with a different value for $C$.)

\Cref{fig:exp1} shows how the proposed method performs using the
theoretical step size and compared to its theory.
We also {compare our initialization to three} other initializations with the same step size.
We consider two different levels of over-factoring, choosing $d=10$ (i.e, $2r$) and $d=6$ (i.e., $r+1$).
All methods perform better for larger $d=10$.
The theory for the proposed initialization underestimates its performance (there may be room for further improvement) but still shows a {stark and consistent advantage compared to the performance of the standard initialization, as well as compared to initializing in the column span of $\A$ but not asymmetrically, or initializing asymmetrically but not in the column span of $\A$.}

\begin{figure}[!ht]
  \centering
  \includegraphics{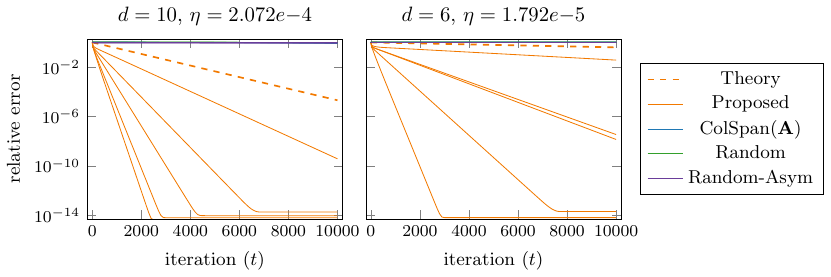}
  \caption{\textbf{Theoretical advantage of proposed initialization.} Gradient descent (non-alternating) for $\mathbf{A} \in \mathbb{R}^{100 \times 100}$ with $\text{rank}(\mathbf{A})=5$, $\sigma_1 = 1$ and $\sigma_r=0.9$. The plot shows five runs with each type of initialization.}
  \label{fig:exp1}
\end{figure}

\Cref{fig:exp2} shows how the three initializations compare with different
step lengths. The advantage of the proposed initialization persists even
with step lengths that are larger than that proposed by the theory, up until
the step length is too large for any method to converge ($\eta=1$).
We emphasize that we are showing standard gradient descent,
not alternating gradient descent. (There is no major difference between
the two in our experience.)

\begin{figure}[!ht]
  \centering
  \includegraphics{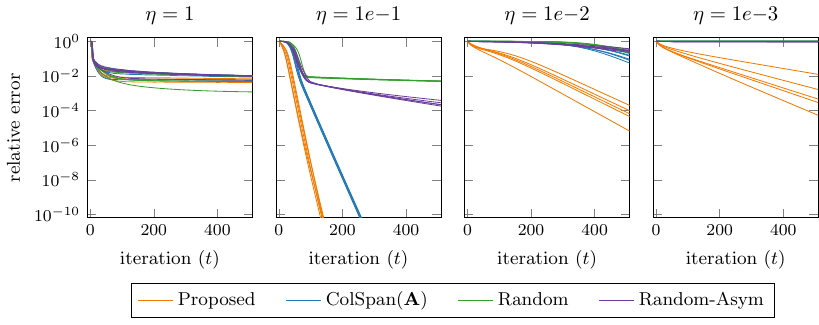}
  \caption{\textbf{Advantage of proposed initialization for different step lengths.} Gradient descent (non-alternating) for $\mathbf{A} \in \mathbb{R}^{100 \times 100}$ with $\text{rank}(\mathbf{A})=5$, $\sigma_1 = 1$ and $\sigma_r=0.9$. The plot shows five runs with each type of initialization.}
  \label{fig:exp2}
\end{figure}

Although the theory requires that $\A$ be exactly rank-$r$, \cref{fig:exp3} shows that the
proposed initialization still maintains its advantage for
noisy problems that are only approximately rank-$r$.

\begin{figure}[!ht]
  \centering
  \includegraphics{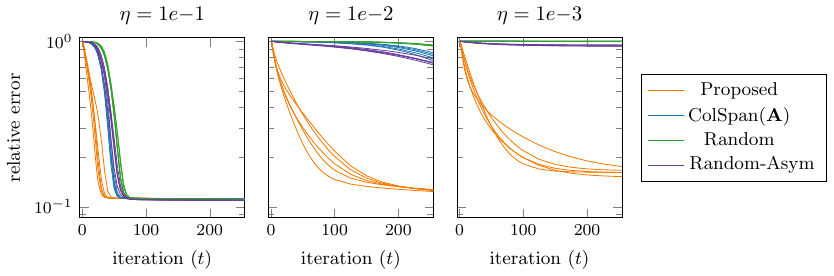}
  \caption{\textbf{Advantage of proposed initialization for noisy problems.} Gradient descent (non-alternating) for $\mathbf{A} \in \mathbb{R}^{100 \times 100}$ with $\text{rank}(\mathbf{A})\approx 5$ (10\% noise), $\sigma_1 = 1$ and $\sigma_r=0.9$. The plot shows five runs with each type of initialization.}
  \label{fig:exp3}
\end{figure}

\bibliographystyle{alpha}

\begin{thebibliography}{GWB{\etalchar{+}}17}

\bibitem[ABC{\etalchar{+}}22]{ahn2022learning}
Kwangjun Ahn, S{\'e}bastien Bubeck, Sinho Chewi, Yin~Tat Lee, Felipe Suarez,
  and Yi~Zhang.
\newblock Learning threshold neurons via the "edge of stability".
\newblock {\em arXiv preprint arXiv:2212.07469}, December 2022.

\bibitem[ACHL19]{arora2019implicit}
Sanjeev Arora, Nadav Cohen, Wei Hu, and Yuping Luo.
\newblock Implicit regularization in deep matrix factorization.
\newblock {\em Advances in Neural Information Processing Systems}, 32, 2019.

\bibitem[BKS16]{bhojanapalli2016dropping}
Srinadh Bhojanapalli, Anastasios Kyrillidis, and Sujay Sanghavi.
\newblock Dropping convexity for faster semi-definite optimization.
\newblock In {\em Conference on Learning Theory}, pages 530--582. PMLR, 2016.

\bibitem[BM03]{burer2003nonlinear}
Samuel Burer and Renato~DC Monteiro.
\newblock A nonlinear programming algorithm for solving semidefinite programs
  via low-rank factorization.
\newblock {\em Mathematical Programming}, 95(2):329--357, 2003.

\bibitem[BM05]{burer2005local}
Samuel Burer and Renato~DC Monteiro.
\newblock Local minima and convergence in low-rank semidefinite programming.
\newblock {\em Mathematical programming}, 103(3):427--444, 2005.

\bibitem[CB22]{chen2022gradient}
Lei Chen and Joan Bruna.
\newblock On gradient descent convergence beyond the edge of stability.
\newblock {\em arXiv preprint arXiv:2206.04172}, June 2022.

\bibitem[CCFM19]{chen2019gradient}
Yuxin Chen, Yuejie Chi, Jianqing Fan, and Cong Ma.
\newblock Gradient descent with random initialization: Fast global convergence
  for nonconvex phase retrieval.
\newblock {\em Mathematical Programming}, 176:5--37, 2019.

\bibitem[CGMR20]{chou2020gradient}
Hung-Hsu Chou, Carsten Gieshoff, Johannes Maly, and Holger Rauhut.
\newblock Gradient descent for deep matrix factorization: Dynamics and implicit
  bias towards low rank.
\newblock {\em arXiv preprint arXiv:2011.13772}, 2020.

\bibitem[CKL{\etalchar{+}}21]{cohen2021gradient}
Jeremy~M Cohen, Simran Kaur, Yuanzhi Li, J~Zico Kolter, and Ameet Talwalkar.
\newblock Gradient descent on neural networks typically occurs at the edge of
  stability.
\newblock {\em arXiv preprint arXiv:2103.00065}, February 2021.

\bibitem[DHL18]{du2018algorithmic}
Simon~S Du, Wei Hu, and Jason~D Lee.
\newblock Algorithmic regularization in learning deep homogeneous models:
  Layers are automatically balanced.
\newblock {\em Advances in neural information processing systems}, 31, 2018.

\bibitem[GHJY15]{ge2015escaping}
Rong Ge, Furong Huang, Chi Jin, and Yang Yuan.
\newblock Escaping from saddle points—online stochastic gradient for tensor
  decomposition.
\newblock In {\em Conference on learning theory}, pages 797--842. PMLR, 2015.

\bibitem[GWB{\etalchar{+}}17]{gunasekar2017implicit}
Suriya Gunasekar, Blake~E Woodworth, Srinadh Bhojanapalli, Behnam Neyshabur,
  and Nati Srebro.
\newblock Implicit regularization in matrix factorization.
\newblock {\em Advances in Neural Information Processing Systems}, 30, 2017.

\bibitem[HKD20]{hong2020generalized}
David Hong, Tamara~G Kolda, and Jed~A Duersch.
\newblock Generalized canonical polyadic tensor decomposition.
\newblock {\em SIAM Review}, 62(1):133--163, 2020.

\bibitem[JCD22]{jiang2022algorithmic}
Liwei Jiang, Yudong Chen, and Lijun Ding.
\newblock Algorithmic regularization in model-free overparametrized asymmetric
  matrix factorization.
\newblock {\em arXiv preprint arXiv:2203.02839}, March 2022.

\bibitem[JJKN17]{jain2017global}
Prateek Jain, Chi Jin, Sham Kakade, and Praneeth Netrapalli.
\newblock Global convergence of non-convex gradient descent for computing
  matrix squareroot.
\newblock In {\em Artificial Intelligence and Statistics}, pages 479--488.
  PMLR, 2017.

\bibitem[KH20]{kolda2020stochastic}
Tamara~G Kolda and David Hong.
\newblock Stochastic gradients for large-scale tensor decomposition.
\newblock {\em SIAM Journal on Mathematics of Data Science}, 2(4):1066--1095,
  2020.

\bibitem[LM00]{lm2000}
Beatrice Laurent and Pascal Massart.
\newblock Adaptive estimation of a quadratic functional by model selection.
\newblock {\em Annals of statistics}, pages 1302--1338, 2000.

\bibitem[LMZ18]{li2018algorithmic}
Yuanzhi Li, Tengyu Ma, and Hongyang Zhang.
\newblock Algorithmic regularization in over-parameterized matrix sensing and
  neural networks with quadratic activations.
\newblock In {\em Conference On Learning Theory}, pages 2--47. PMLR, 2018.

\bibitem[RV09]{rudelson2009smallest}
Mark Rudelson and Roman Vershynin.
\newblock Smallest singular value of a random rectangular matrix.
\newblock {\em Communications on Pure and Applied Mathematics: A Journal Issued
  by the Courant Institute of Mathematical Sciences}, 62(12):1707--1739, 2009.

\bibitem[SWW17]{sanghavi2017local}
Sujay Sanghavi, Rachel Ward, and Chris~D White.
\newblock The local convexity of solving systems of quadratic equations.
\newblock {\em Results in Mathematics}, 71:569--608, 2017.

\bibitem[TBS{\etalchar{+}}16]{tu2016low}
Stephen Tu, Ross Boczar, Max Simchowitz, Mahdi Soltanolkotabi, and Ben Recht.
\newblock Low-rank solutions of linear matrix equations via {Procrustes} flow.
\newblock In {\em International Conference on Machine Learning}, pages
  964--973. PMLR, 2016.

\bibitem[Ver10]{vershynin2010introduction}
Roman Vershynin.
\newblock Introduction to the non-asymptotic analysis of random matrices.
\newblock In {\em Compressed Sensing}, pages 210--268. Cambridge University
  Press, may 2010.

\bibitem[WCZT21]{wang2021large}
Yuqing Wang, Minshuo Chen, Tuo Zhao, and Molei Tao.
\newblock Large learning rate tames homogeneity: Convergence and balancing
  effect.
\newblock {\em arXiv preprint arXiv:2110.03677}, 2021.

\bibitem[YD21]{ye2021global}
Tian Ye and Simon~S Du.
\newblock Global convergence of gradient descent for asymmetric low-rank matrix
  factorization.
\newblock {\em Advances in Neural Information Processing Systems},
  34:1429--1439, 2021.

\bibitem[ZL16]{zheng2016convergence}
Qinqing Zheng and John Lafferty.
\newblock Convergence analysis for rectangular matrix completion using
  {Burer-Monteiro} factorization and gradient descent.
\newblock {\em arXiv preprint arXiv:1605.07051}, 2016.

\bibitem[ZWL15]{zhao2015nonconvex}
Tuo Zhao, Zhaoran Wang, and Han Liu.
\newblock Nonconvex low rank matrix factorization via inexact first order
  oracle.
\newblock {\em Advances in Neural Information Processing Systems},
  458:461--462, 2015.

\end{thebibliography}
\newcommand{\etalchar}[1]{$^{#1}$}

\clearpage
\appendix

\section{Non-asymptotic singular value bounds}
\label{sec:non-asymp-sing}

\begin{prop}[Theorem 1.1 of \cite{rudelson2009smallest}]
\label{prop:smallest}
  Let $\A$ be an $d \times r$ matrix, $d \geq r$, whose entries are
  independently drawn from $\mathcal{N}(0,1)$. Then
  for every $\tau \geq 0$,
  \begin{displaymath}
\Prob\left( \sigma_r(\A) \leq \tau (\sqrt{d} - \sqrt{r-1}) \right) \leq (C_1 \tau)^{d - r + 1} + e^{-C_2 d}
    \end{displaymath}
where $C_1 , C_2 > 0$ are universal constants.
\end{prop}

\begin{prop}[\cite{vershynin2010introduction}]
\label{prop:v}
  Let $\A$ be an $d \times r$ matrix whose entries are
  independently drawn from $\mathcal{N}(0,1)$. Then
  for every $t \geq 0$, with probability at least
  $1-\exp(-t^2/2)$, we have
  \begin{displaymath}
\sigma_r(\A) \geq   \sqrt{d} - \sqrt{r} - t
    \end{displaymath}
    and for every $t \geq 0$, with probability at least
  $1-\exp(-t^2/2)$, we have
  \begin{displaymath}
\sigma_1(\A) \leq   \sqrt{d} + \sqrt{r} + t
    \end{displaymath}
\end{prop}

{\color{magenta}

\begin{prop}[Lemma 1 of \cite{lm2000}]
\label{prop:lm2000}
Let $a_1, \dots, a_n$ be nonnegative, and set 
$$
\| \Vc{a} \|_{\infty} = \sup_{i=1,\dots, n} |a_i|, \quad \| \Vc{a} \|_2^2 = \sum_{i=1}^n a_i^2
$$
For independent and identically-distributed $Z_i \sim {\cal N}(0,a_i),$ let
$$
X = \sum_{i=1}^n (Z_i^2 - a_i).
$$
Then the following inequalities hold for any positive t:
\begin{eqnarray}
\Prob(X \geq 2 \| \Vc{a} \|_2 \sqrt{t} + 2 \| \Vc{a} \|_{\infty} t) &\leq e^{-t} \nonumber \\
\Prob(X \leq -2 \| \Vc{a} \|_2 \sqrt{t}) &\leq e^{-t}
\end{eqnarray}
\end{prop}

}

\section{Proof of \cref{lemma1}}
\label{sec:proof-lemma1}

{ The proof of \cref{lemma1} is a direct calculation:
\begin{align*}
  f(\Xt+,\Yt) &= \frac{1}{2} \fnrm{\A-\Xt+\Yt'}^2 \\
  &= \frac{1}{2}\fnrm{\A - (\Xt - \eta \Gt)\Yt'}^2\\
  &= \frac{1}{2}\fnrm{\A - \Xt\Yt' + \eta \Gt\Yt'}^2\\
  &= \frac{1}{2} \fnrm{\A-\Xt\Yt'}^2 + \frac{1}{2} \fnrm{\eta \Gt\Yt'}^2
  - \eta \trace[ (\Xt\Yt'-\A) \prn(\Gt\Yt')']\\
  &= \frac{1}{2} \fnrm{\A-\Xt\Yt'}^2 + \frac{1}{2} \fnrm{\eta \Gt\Yt'}^2
  - \eta \trace[ \underbrace{(\Xt\Yt'-\A) \Yt}_{\Gt} (\Gt)^{\Tr}]\\
  &= f(\Xt,\Yt) + \frac{\eta}{2} \fnrm{\Gt\Yt'}^2 - \eta \fnrm{\Gt}^2\\
  &\leq f(\Xt,\Yt) + \frac{\eta^2}{2} \fnrm{\Gt}^2 \nrm{\Yt}^2 - \eta \fnrm{\Gt}^2\\
  &\leq f(\Xt,\Yt) + \frac{\eta}{2} \fnrm{\Gt}^2 - \eta \fnrm{\Gt}^2\\
  &\leq f(\Xt,\Yt) - \frac{\eta}{2} \fnrm{\Gt}^2. \qedhere
\end{align*}
}  

\section{Proof of \cref{prop:big}}
\label{sec:proof1}
First, observe that by assumption, $\nrm{ \Xi}^2, \nrm{\Yi }^2 \leq \frac{9}{16 \eta} \leq \frac{1}{\eta}.$  Now,  suppose that that $\nrm{ \Xi }^2, \nrm{ \Yi }^2  \leq \frac{9}{16\eta}$ and
$\nrm{ \Xt }^2, \nrm{ \Yt }^2  \leq \frac{1}{\eta}$ for $t=0,\dots T-1$, and 
$1 \leq T \leq \lfloor \frac{1}{32\eta^2 {f_0}} \rfloor$.
Then by \cref{lemma1}, 
\begin{align}
     \sum_{t=0}^{T-1} \fnrm***{ \D{\X} f(\Xt, \Yt) }^2 \leq \frac{2}{\eta} f_0.
\end{align}

Hence,
\begin{align}
     \nrm{ \X{T} - \Xi }  &\leq \eta \nrm***{ \sum_{t=0}^{T-1} \D{\X} f(\Xt, \Yt) }  \leq \eta \fnrm***{ \sum_{t=0}^{T-1} \D{\X} f(\Xt, \Yt) } \nonumber \\
     &\leq \eta \sqrt{ \sum_{t=0}^{T-1} \nrm{ \D{\X} f(\Xt, \Yt) }^2_F } \leq \eta \sqrt{\frac{2}{\eta} f_0}  = \sqrt{2 \eta f_0}.
\end{align}
Then, for $T \leq T_*,$ $\nrm{ \X{T} } \leq \nrm{\Xi} + \nrm{ \X{T} - \Xi } \leq \frac{3}{4\sqrt{\eta}} + \sqrt{2T \eta f_0} \leq \frac{1}{\sqrt{\eta}}.$  It follows that $\nrm{ \Xt }^2 \leq \frac{1}{\eta}$ for $t=0,\dots T$.  Using  \cref{lemma1} again, repeating the same argument, 
\begin{align*}
    \nrm{ \Y{t} } &\leq \frac{1}{\sqrt{\eta}}, \quad t= 0, \dots, T.
\end{align*}
Iterate the induction until $T = T_* = \lfloor \frac{1}{32\eta^2 {f_0}} \rfloor,$ to obtain $\nrm{ \Xt }^2, \nrm{ \Yt }^2 \leq \frac{1}{\eta}$ for $t =  1,\dots, T_*$.

Because $\| \X_{T} - \Xi \| \leq \sqrt{2 \eta T f_0}$ for $T \leq T_* = \lfloor \frac{1}{32\eta^2 {f_0}} \rfloor,$
\begin{align*}
\sigma_r(\X{T}) &\geq \sigma_r(\Xi) - \nrm{ \X{T} - \Xi }; \quad \quad \sigma_1(\X{T}) \leq \sigma_1(\Xi) + \nrm{ \X{T} - \Xi }. \nonumber 
\end{align*}
A similar argument applies to achieve the stated bounds for $\sigma_r(\Y{T})$ and $\sigma_1(\Y{T})$.

\section{Proof of \cref{prop:initialize}}
\label{sec:proof2}
{Recall the initialization in \cref{ass:init}:
$\RX,\RY \in \Real[n,d]$ are random matrices with \iid $\Normal(0,1/d)$ and $\Normal(0,1/n)$ entries, respectively.
  Fix $C \geq 1$, $\nu < 1,$ and $D \leq \frac{C}{9}\nu,$ and let 
  \begin{equation}\nonumber
    \Xi = \frac{1}{\eta^{1/2} \, C \, \sigma_1(\A)} \, \A \RX,
    \qtext{and} 
    \Yi = \eta^{1/2} \, D \, \sigma_1(\A) \, \RY.    
  \end{equation}
  The factor matrices each have $d \geq r$ columns.
}
Write the SVD $\A = \Mx{U}{m \times r} \Mx{\Sigma}{r \times r} \Mx{V}{r \times n}'$
so that 
$\A \RX =  \Mx{U}{m \times r} \Mx{\Sigma}{r \times r} (\Mx{V}' \RX)$.
Note that $\Mx{V}' \RX \in \mathbb{R}^{r \times d}$ has i.i.d.\@ Gaussian entries $\Normal(0, \frac{1}{d})$. 
By \cref{prop:smallest}, with probability at least $1 - (C_1 \tau)^{d - r + 1} - e^{-C_2 d}$,
\begin{equation}
  \label{eq1a}
 \sigma_r(\Mx{V}' \RX) \geq \tau \left( 1 - \frac{\sqrt{r-1}}{\sqrt{d}} \right)
\end{equation}
  { 
  On the other hand,  \cref{prop:v} implies that with probability at least $1 - e^{-r/2} - e^{-d/2}$, 
 \begin{align}
 \label{e1}
\sigma_1(\Mx{V}' \RX) &\leq \prn**(1+ \frac{2\sqrt{r}}{\sqrt{d}}) \leq 3
\end{align}
and
\begin{align}
\label{e2}
\sigma_1(\RY) &\leq \prn**(1+ \frac{2\sqrt{d}}{\sqrt{n}}) \leq 3.
\nonumber
\end{align}
If \cref{eq1a,e1} hold, then
\begin{equation}\label{e3}
  \frac{\tau \left( 1 - \frac{\sqrt{r-1}}{\sqrt{d}} \right)}{\sqrt{\eta}C \sigma_1(\A)} \sigma_r(\A)
  \leq \sigma_r(\Xi)
  \leq \sigma_1(\Xi)
  \leq   \frac{3}{\sqrt{\eta}C}.
\end{equation}
Now, note that \cref{e3} implies,
\begin{align}
2 f(\Xi, \Yi) &= \fnrm{ \A - \Xi \Yi' }^2  \nonumber \\
&\leq (\fnrm{ \A  } + \fnrm{ \Xi \Yi' })^2 \nonumber \\
&\leq \prn*(\fnrm{ \A } + \frac{D}{C} \| \Mx{\Sigma}{r \times r} \| \sigma_1 (\Mx{V}' \RX) \sigma_1 (\RY) )^2\nonumber \\
&\leq \prn*(\fnrm{ \A } + \nu \sigma_1(\A))^2 \nonumber \\
&\leq (1 + \nu)^2 \fnrm{ \A }^2 \nonumber
\end{align}

Finally, we can produce a lower bound as follows on $f(\Xi, \Yi)$ as follows:
\begin{align}
2 f(\Xi, \Yi) = \fnrm*{ \A - \Xi \Yi' }^2  &= \fnrm{ \A }^2 + \fnrm{ \Xi \Yi' }^2 - 2 \text{Tr}(\Yi^T \A \Xi) \nonumber \\
&\geq \fnrm{ \A }^2 - 2 \text{Tr}(\Yi' \A' \Xi) \nonumber \\
&= \fnrm{ \A }^2 - 2 \frac{D}{C} \text{Tr}(\RY' \A' A \RX) \nonumber \\
&= \fnrm{ \A }^2 - 2 \frac{D}{C} \text{Tr}
(\RY' \Mx{V} \Mx{\Sigma} \Mx{U}'\Mx{U} \Mx{\Sigma} (\Mx{V}' \RX)) \nonumber \\
&= \fnrm{ \A }^2 - 2 \frac{D}{C} \text{Tr}
(\underbrace{\RY' \Mx{V}}_{\Mx{G}{1}'} \underbrace{\Mx{\Sigma}^2}_{\Mx{\Lambda}}
\underbrace{\Mx{V}' \RX}_{\Mx{G}{2}}) \nonumber %
\end{align}
where $\Mx{\Lambda} := \Mx{\Sigma}^2$ is an $ \times r$ diagonal matrix with diagonal entries $\sigma_1(\A)^2, \dots, \sigma_r(\A)^2$, and due to the invariance of Gaussian measure, $G_1 := \Mx{V}'\RX$ and $\Mx{G}{2} :=  \Mx{V}'\RY$ are independent $r \times d$ Gaussian matrices with i.i.d. ${\cal N}(0,1/d)$ and ${\cal N}(0,1/n)$ entries (denoted by $g_1(j,k)$ and $g_2(j,k)$), respectively.

Observe now that
\begin{align}
  \text{Tr}(\Mx{G}{2}' \Mx{\Lambda} \Mx{G}{1})
  &= \sum_{j=1}^r \sum_{k=1}^d  \sigma^2_k g_2(j,k) g_1(j,k)
  \nonumber \\
&= \frac{1}{4} \sum_{j=1}^r \sum_{k=1}^d  \sigma^2_k ( g_2(j,k)+ g_1(k,j))^2 -  \frac{1}{4} \sum_{j=1}^r \sum_{k=1}^d  \sigma^2_k ( g_2(j,k) - g_1(k,j))^2 \nonumber .
\end{align}
The first equality is a calculation, and the second equality simply applies the numerical identity 
\begin{equation*}
  ab = \frac{1}{4}((a+b)^2 - (a-b)^2), \quad \forall a,b.
\end{equation*}
Therefore, we find that 
\begin{equation*}
  2 f(\X_0, \Y_0) \geq \fnrm{ \A }^2 - \frac{\nu}{9}  \sum_{j=1}^d \sum_{k=1}^r  Z^2_{j,k} - \frac{\nu}{9} \sum_{j=1}^d \sum_{k=1}^r  (Z'_{j,k})^2
\end{equation*}
where $Z_{j,k}$ are independent ${\cal N}(0, \sigma_k^2 (\frac{1}{n} + \frac{1}{d}))$ variables.  Similarly, $Z_{j,k}'$ are independent ${\cal N}(0, \sigma_k^2 (\frac{1}{n} + \frac{1}{d}))$ variables.   Applying  \cref{prop:lm2000} lower bound gives
\begin{equation*}
  \Prob \prn+{
    \sum_{j=1}^d \sum_{k=1}^r  Z^2_{j,k} \geq
    d \prn+(\frac{1}{d} {+} \frac{1}{n}) \fnrm{\A}^2
    + 2 \sqrt{d \prn+(\frac{1}{d} {+} \frac{1}{n})^2 \sum_{i=1}^r  \sigma_i^4 }  \sqrt{\xi}
    + 2 \prn+(\frac{1}{d} {+} \frac{1}{n}) \nrm{\A}^2 \xi
    }
  \leq e^{-\xi},
\end{equation*}
and the exact same bound with $Z'_{j,k}$ replacing $Z_{j,k}$.
Recalling $d \leq n$, we have
\begin{multline*}
   d \prn+(\frac{1}{d} {+} \frac{1}{n}) \fnrm{\A}^2
    + 2 \sqrt{d \prn+(\frac{1}{d} {+} \frac{1}{n})^2 \sum_{i=1}^r  \sigma_i^4 }  \sqrt{\xi}
    + 2 \prn+(\frac{1}{d} {+} \frac{1}{n}) \nrm{\A}^2 \xi \\
    \begin{aligned}
      \leq & 2 \fnrm{\A}^2 + \frac{4}{\sqrt{d}} \nrm{ \A } \fnrm{ \A } \sqrt{\xi} + \frac{4}{d}\nrm{ \A }^2 \xi \\
      \leq & 2 \fnrm{\A}^2 + \frac{4}{\sqrt{d}} \fnrm{ \A }^2 \sqrt{\xi} + \frac{4}{d}\nrm{ \A }^2 \xi.
    \end{aligned}
\end{multline*}
Therefore,
\begin{equation*}
  \Prob \prn+{
    \sum_{j=1}^d \sum_{k=1}^r  (Z)^2_{j,k} \geq
2 \fnrm{\A}^2 + \frac{4}{\sqrt{d}} \fnrm{ \A }^2 \sqrt{\xi} + \frac{4}{d}\nrm{ \A }^2 \xi}  
  \leq e^{-\xi}.
\end{equation*}
Applying this bound with the particular value $\xi = d$ to both $Z$ and $Z'$, and taking a union bound, we have that with probability at most $2 e^{-d}$,
\begin{align}
2 f(\X_0, \Y_0) &\geq \fnrm{ \A }^2 - \frac{20\nu}{9} \fnrm{ \A }^2 \nonumber \\
&\geq (1 - 2.3\nu) \fnrm{ \A }^2,
\end{align}
finishing the proof
}

\section{Proof of  \cref{thm:big}}
\label{sec:proof4}

Set $\beta_1 = \beta$ as in \cref{b:define}. Set ${f_0}_{(1)} = f_0.$

   By \cref{cor:main},  iterating \cref{ass:agd} for $T_1 = \lfloor \frac{\beta_1}{8 \eta^2 {f_0}_{(1)}} \rfloor$ iterations with step-size 
\begin{displaymath}
   \eta \leq \frac{\beta_1}{\sqrt{32 {f_0}_{(1)} \log(1/\epsilon)}} 
 \end{displaymath}
 guarantees that 
   \begin{align}
\frac{1}{2} \fnrm{ \A - \X_{T_1} \Y_{T_1}' }^2 &\leq {f_0}_{(2)} := \epsilon {f_0}_{(1)}; \nonumber \\
\nrm{ \sigma_r(\X_{T_1}) }^2 &\geq \frac{1}{4}  \frac{\beta_1}{\eta}. \nonumber 
\end{align}
This means that at time $T_1$, we can restart the analysis, and appeal again to  \cref{prop:big} with modified parameters
\begin{itemize}
    \item $f(\X_{T_1}, \Y_{T_1}) \leq {f_0}_2 := \epsilon {f_0}_1,$
    \item $\beta_2 := \frac{\beta_1}{4}.$
\end{itemize}
\Cref{cor:main} again guarantees that provided 
\begin{equation}
    \label{eps}
\eta \leq \frac{\beta_2}{\sqrt{32 {f_0}_{(2)} \log(1/\epsilon)}} 
= \frac{1}{4\sqrt{\epsilon}} \frac{\beta_1}{\sqrt{32 {f_0}_{(1)} \log(1/\epsilon)}}
\end{equation}
then $f(\X_{T_1 + T_2}, \Y_{T_1 + T_2}) \leq \epsilon f(\X_{T_1}, \Y_{T_1}) \leq \epsilon^2 f(\X_{0}, \Y_{0}) $ where 
\begin{align}
T_2 &= \frac{T_1}{4\epsilon} .
\end{align}
We have that
\cref{eps} is satisfied by assumption as we assume $\epsilon \leq \frac{1}{16}.$  Repeating this inductively, we find that after
$T = T_1 + \dots + T_k = T_1 \sum_{\ell=0}^{k-1} (\frac{1}{4\epsilon})^{\ell}  \leq T_1 (\frac{1}{4\epsilon})^{k}$ iterations, we are guaranteed that
$f(\X_T, \Y_T) \leq \epsilon^k f(\Xi, \Yi).$
This is valid for any $T \in \mathbb{N}$ because we may always apply \cref{prop:big} in light of summability and $C \geq 8$: for any $t$,
\begin{equation*}
\begin{aligned}[b]
    \sigma_1(\X_t) &\leq \sigma_1(\Xi) + \sqrt{\eta} \sum_{j=1}^k \sqrt{2 T_k {f_0}_{(k)}} \nonumber \\
  &\leq \frac{3}{8 \sqrt{\eta}} + \frac{1}{\sqrt{\eta}} \sum_{j=1}^k \sqrt{2 (1/(4 \epsilon))^j \frac{\beta_1}{8  {f_0}_1} \epsilon^j {f_0}_{(1)}} \nonumber \\
  &\leq \frac{3}{8 \sqrt{\eta}} + \frac{\sqrt{\beta}}{2\sqrt{\eta}} \sum_{j=1}^k(1/2)^j \nonumber \\
  &\leq \frac{3 + 4 \sqrt{\beta}}{8 \sqrt{\eta}}  \leq \frac{1}{2 \sqrt{\eta}}.
  \end{aligned}\qedhere
  \end{equation*}

\end{document}